\documentclass{article}

\PassOptionsToPackage{numbers}{natbib}

\usepackage[preprint]{neurips_2025}

\usepackage{bbm}
\usepackage{enumitem}
\usepackage{makecell}
\usepackage{listings}
\usepackage{xcolor}
\usepackage{algorithm}
\usepackage{algorithmic}

\usepackage{amsmath}
\usepackage{amssymb}
\usepackage{mathtools}
\usepackage{amsthm}

\allowdisplaybreaks

\theoremstyle{plain}
\newtheorem{theorem}{Theorem}[section]
\newtheorem{example}[theorem]{Example}

\newtheorem{lemma}[theorem]{Lemma}
\newtheorem{corollary}[theorem]{Corollary}
\theoremstyle{definition}
\newtheorem{definition}[theorem]{Definition}

\theoremstyle{remark}

\newcommand{\mat}[1]{\boldsymbol{\mathbf{#1}}}
\newcommand{\ind}{\mathbbm{1}}
\newcommand{\minmax}{\mathop{\text{min}/\text{max}}}
\DeclareMathOperator{\vect}{vec}

\definecolor{codeblue}{rgb}{0.13, 0.13, 1.0}
\definecolor{codegreen}{rgb}{0.0, 0.5, 0.0}
\definecolor{codegray}{rgb}{0.5, 0.5, 0.5}
\definecolor{backcolour}{rgb}{0.95, 0.95, 0.92}

\lstdefinestyle{pythonstyle}{
    backgroundcolor=\color{backcolour},
    commentstyle=\color{codegreen},
    keywordstyle=\color{codeblue}\bfseries,
    numberstyle=\tiny\color{codegray},
    stringstyle=\color{codegreen},
    basicstyle=\ttfamily\footnotesize,
    breakatwhitespace=false,
    breaklines=true,
    captionpos=b,
    keepspaces=true,
    numbers=left,
    numbersep=5pt,
    showspaces=false,
    showstringspaces=false,
    showtabs=false,
    tabsize=4,
    frame=tb,
    language=Python
}

\usepackage[utf8]{inputenc} 
\usepackage[T1]{fontenc}    
\usepackage{hyperref}       
\usepackage{url}            
\usepackage{booktabs}       
\usepackage{amsfonts}       
\usepackage{nicefrac}       
\usepackage{microtype}      
\usepackage{xcolor}         

\title{Formal Verification of Markov Processes with Learned Parameters}

\author{%
  Muhammad Maaz \\
  University of Toronto \\
  \texttt{m.maaz@mail.utoronto.ca}
  \And
  Timothy C. Y. Chan \\
  University of Toronto \\
  \texttt{tcychan@mie.utoronto.ca}
}

\begin{document}

\maketitle

\begin{abstract}
  We introduce the problem of formally verifying properties of Markov processes where the parameters are given by the output of machine learning models. For a broad class of machine learning models, including linear models, tree-based models, and neural networks, verifying properties of Markov chains like reachability, hitting time, and total reward can be formulated as a bilinear program. We develop a decomposition and bound propagation scheme for solving the bilinear program and show through computational experiments that our method solves the problem to global optimality up to 100x faster than state-of-the-art solvers. To demonstrate the practical utility of our approach, we apply it to a real-world healthcare case study.
  Along with the paper, we release \texttt{markovml}, an open-source tool for building Markov processes, integrating pretrained machine learning models, and verifying their properties, available at \href{https://github.com/mmaaz-git/markovml}{https://github.com/mmaaz-git/markovml}.
\end{abstract}

\section{Introduction}
\label{sec:intro}

Markov processes are fundamental mathematical models used across computer science, operations research, engineering, and healthcare \cite{stewart2021introduction}. In computer science, they capture the behavior of probabilistic systems such as hardware, communication protocols, or autonomous agents \cite{baier2008principles}. In engineering, they are used to analyze degradation and failure in complex machinery \cite{rausand2003system}. In healthcare, they model patient transitions between clinical states and underpin cost-effectiveness analyses of medical interventions \cite{sonnenberg1993markov}.

As machine learning (ML) becomes increasingly integrated into real-world systems, Markov models are evolving to incorporate heterogeneous, data-driven parameters. For instance, ML models can estimate failure rates from sensor data or patient-specific transition probabilities based on clinical features \cite{mertens2022microsimulation}. This gives rise to a new class of Markov processes where parameters are not fixed but are instead \textit{learned functions}.

Despite this growing trend, there is limited work on rigorously analyzing such systems. In healthcare, most studies rely on Monte Carlo simulation \cite{krijkamp2018microsimulation}, which supports subgroup analyses (e.g., patients over 60) but cannot provide formal guarantees -- an important limitation in high-stakes domains like medicine or safety-critical infrastructure.

In this paper, we present a new framework for the formal verification of Markov processes with ML-based parameters. Our key insight is that properties of such systems can be encoded as \emph{bilinear programs}, allowing us to obtain exact results with formal guarantees. While our primary motivation is healthcare, the approach generalizes to any application where learned models feed into Markov processes. This enables us to rigorously answer questions such as: Given a bound on the input, what is the worst-case probability of reaching a failure state? Is the failure rate for machines with certain properties guaranteed to remain below 0.01\%? For a clinical subgroup, is the expected treatment cost within a government threshold?

To summarize, our main contributions are:
\begin{enumerate}[noitemsep,topsep=0pt]
\item We introduce a general framework for formally verifying properties of Markov processes whose parameters are given by ML models.
\item We show that for a broad class of models, including linear models, tree ensembles, and ReLU-based neural networks, verification problems can be expressed as (mixed-integer) bilinear programs.
\item We develop a novel decomposition and bound propagation method that significantly accelerates global solution times, outperforming existing solvers by orders of magnitude.
\end{enumerate}

\section{Related Work}
\label{sec:related}

\paragraph{Formal Verification of Markov Processes}
Probabilistic model checking verifies properties of systems modeled by Markov chains \cite{hansson1994logic, baier2008principles}, using tools like PRISM and others \cite{kwiatkowska2002prism, katoen2005markov, hermanns2000markov, sen2005statistical, younes2005ymer, lassaigne2002approximate}. More recent techniques allow parameters to be intervals \cite{sen2006model, delahaye2015parameter, petrucci2018parameter}, rational functions \cite{chen2013model, junges2020parameter, junges2024parameter}, or distributions \cite{badings2023efficient}. We extend this further by allowing parameters defined by ML models, and release our own tool, \texttt{markovml}.

\paragraph{Markov Processes with Uncertain Parameters}
Markov processes with parameters in uncertainty sets or defined by differentiable functions have been well-studied \cite{blanc2008markov, de2014sensitivity, marbach2001simulation, marbach2003approximate}, as have Markov \textit{decision} processes with uncertain parameters \cite{el2005robust,iyengar2005robust,goyal2023robust,grand2024convex}. \citet{goh2018data} develop a value iteration algorithm for uncertain transitions, and \citet{chan2024exact} study the inverse problem of characterizing parameters satisfying a reward threshold. In contrast, we study uncertainty arising from ML-predicted parameters.

\paragraph{Formal Verification of ML Models}
Formal verification of ML models aims to prove properties such as robustness or output bounds; a well-known benchmark is ACAS Xu \cite{owen2019acas}. ReLU neural networks can be encoded as satisfiability modulo theories (SMT) or mixed-integer linear programming (MILP) problems, enabling verification via tools like Reluplex \cite{katz2017reluplex}, a modified simplex method that spurred substantial follow-up work \cite{tjeng2017evaluating, cheng2017maximum, anderson2020strong, tjandraatmadja2020convex, kronqvist2021between, anh2022neural}. State-of-the-art methods such as $\alpha, \beta$-CROWN combine bound propagation, branch-and-bound, and GPU acceleration \cite{wang2021beta,zhang22babattack,zhang2022general,kotha2023provably}. We leverage these formulations to embed ML models directly into our optimization framework.

\paragraph{Subgroup Analysis in Healthcare}
Our work is most directly motivated by subgroup analysis in healthcare cost-benefit evaluations using Markov microsimulation models, which capture heterogeneous patient trajectories over time \cite{krijkamp2018microsimulation}. These models compute metrics like the incremental cost-effectiveness ratio (ICER) and net monetary benefit (NMB) \cite{sonnenberg1993markov}, using individual-level transition probabilities derived from so-called risk scores -- typically logistic regression, although other models are sometimes used \cite{mertens2022microsimulation, wilde2019cost, lee2020health, breeze2017cost}. While existing approaches rely on Monte Carlo simulation, we instead provide an exact, non-simulation-based framework that supports a broad class of machine learning models, aligning with the shift toward more sophisticated analytics in healthcare.

\paragraph{Bilinear Programming}
Bilinear programs are NP-hard non-convex quadratic problems in which the objective or constraints contain bilinear terms \cite{horst2013handbook}. Global optima can be computed via branch-and-bound \cite{horst2013handbook}, aided by convex relaxations like McCormick’s envelope \cite{mccormick1976computability}, and are supported by modern solvers such as Gurobi \cite{gurobi}. While tighter relaxations or MILP reformulations exist under certain structural conditions \cite{dey2019new, horst2013handbook}, our problem does not satisfy these, necessitating a new approach. Moreover, we find that Gurobi's default approach performs poorly on our class of problems. In contrast, our method accelerates solution time by orders of magnitude while still leveraging Gurobi’s global optimality guarantees.

\section{Problem Formulation}
\label{sec:formulation}

\paragraph{Notation}
\label{subsec:notation}

Vectors are lowercase bold, e.g., $\mat{x}$, with $i$-th entry $x_i$, and matrices by uppercase bold letters, e.g., $\mat{M}$ with $(i,j)$-th entry $M_{ij}$. The identity matrix is denoted by $\mat{I}$, the vector of all ones by $\mat{1}$, with dimensions inferred from context. The set of integers from $1$ to $n$ is denoted by $[n]$.

\paragraph{Preliminaries}
\label{subsec:preliminaries}

A (discrete-time, finite-state) \emph{Markov chain} with $n$ states, indexed by $[n]$, is defined by a row-stochastic transition matrix $\mat{P} \in \mathbb{R}^{n \times n}$, where $P_{ij}$ is the probability of transitioning from state $i$ to state $j$, and a stochastic initial distribution vector $\mat{\pi} \in \mathbb{R}^n$, where $\pi_i$ is the probability of starting in state $i$. Furthermore, if we assign rewards to each state, we call this a \emph{Markov reward process}. A Markov reward process has a reward vector $\mat{r} \in \mathbb{R}^n$, where $r_i$ is the reward for being in state $i$ for one period. A state is  \emph{absorbing} if it cannot transition to any other state, and \emph{transient} otherwise.

Below, we recount three of the key properties commonly computed in practice \cite{puterman2014markov}.

\begin{definition}[Reachability]
    \label{def:reachability}
    The probability of eventually reaching a set of states $S \subseteq [n]$, from a set $T \subseteq [n]$ of transient states equal to the complement of $S$, assuming that the chain will reach $S$ from $T$ with probability 1, is given by $\mat{\tilde{\pi}}^\top (\mat{I} - \mat{Q})^{-1} \mat{R} \mat{1}$, where $\mat{Q}$ is the transition matrix restricted to $T$, $\mat{R}$ is the transition matrix from $T$ to $S$, and $\mat{\tilde{\pi}}$ is the initial distribution over $T$.
\end{definition}

\begin{definition}[Expected hitting time]
    \label{def:hitting_time}
    The expected number of steps to eventually reach a set of states $S \subseteq [n]$, from a set $T \subseteq [n]$ of transient states from a set $T \subseteq [n]$ of transient states equal to the complement of $S$, assuming that the chain will reach $S$ from $T$ with probability 1, is given by $\mat{\tilde{\pi}}^\top (\mat{I} - \mat{Q})^{-1} \mat{1}$, where $\mat{Q}$ is the transition matrix restricted to $T$, and $\mat{\tilde{\pi}}$ is the initial distribution over $T$.
\end{definition}

\begin{definition}[Total infinite-horizon discounted reward]
    \label{def:discounted_reward}
    The total infinite-horizon discounted reward, with a discount factor $\lambda \in (0, 1)$ is given by $\sum_{t=0}^\infty \lambda^t \mat{\pi}^\top \mat{P}^t \mat{r} = \mat{\pi}^\top \left( \mat{I} - \lambda \mat{P} \right)^{-1} \mat{r}$.
\end{definition}

For each of the above quantities, we can restrict the analysis to a single state, e.g., for reachability starting from a state $i$, we simply set $\tilde{\pi}_i=1$ and for $j \neq i$, set $\tilde{\pi}_j=0$.

These three quantities enable rich analysis of Markov processes and the systems they model. Reachability can be used to verify the probability of failure in a system. Expected hitting time can be used to compute the expected time to failure, or life expectancy of a person. Total reward can be used to compute the resource consumption of a system or total cost of a drug.

We wish to \textit{verify} these properties, namely finding their maximum or minimum. If the Markov process' parameters are fixed, as is common in model checking, these quantities can be computed exactly by solving a linear system. However, in our case, the parameters are given by ML models, and so we will formulate an optimization problem to derive bounds on these quantities. As the three quantities have similar formulations, we will proceed in the rest of the paper by studying the total reward. It will be easy to modify results for the other two quantities.

\subsection{Embedding ML Models}
We now consider $\mat{\pi}, \mat{P}, \mat{r}$ as functions of a \emph{feature vector} $\mat{x} \in \mathbb{R}^m$, where $m$ is the number of features. We encode the relationship from $\mat{x}$ to $\mat{\pi}, \mat{P}, \mat{r}$ via a set of functions, $f_1, f_2, \ldots, f_{k_f}$, where $k_f$ is the number of functions. Each function $f_i: \mathbb{R}^m \rightarrow \mathbb{R}^{\ell_i}$ for $i = 1, 2, \ldots, k_f$ takes a feature vector and outputs a vector $\mat{\theta_i}$ (e.g., a classifier may output class probabilities for three classes). We concatenate these vectors to form the output vector $\mat{\theta} \in \mathbb{R}^\ell$, where $\ell = \sum_{i=1}^{k_f} \ell_i$.

Our key assumption on the functions will be that they are \emph{mixed-integer linear representable} (MILP-representable), meaning that the relationship between the inputs and outputs can be expressed using linear constraints and binary variables. This is a broad class of functions that includes, e.g., piecewise linear functions (including simple ``if-then'' rules), linear and logistic regression, tree-based models, and neural networks with ReLU activations

It is important to note that in the case of logistic regression and neural network classifiers, there is a non-linearity introduced by the softmax or logistic function. Typically in the formal verification literature, it suffices to check the logits, not the actual probabilities. However, in our problem, we need the probabilities as they are inputs into our Markov process. Practically, we handle this issue in this paper and in our software using Gurobi's built-in \textit{spatial branch-and-bound techniques} for nonlinear functions -- essentially dynamic piecewise linear approximations.

Next, the output vector $\mat{\theta}$ is linked to $\mat{\pi}, \mat{P}$ and $\mat{r}$ through affine equalities: $\mat{\pi} = \mat{A_\pi} \mat{\theta} + \mat{b_\pi}$, $\vect(\mat{P}) = \mat{A_P} \mat{\theta} + \mat{b_P}$, and $\mat{r} = \mat{A_r} \mat{\theta} + \mat{b_r}$, for fixed $\mat{A_\pi} \in \mathbb{R}^{n \times \ell}, \mat{A_P} \in \mathbb{R}^{n^2 \times \ell}, \mat{A_r} \in \mathbb{R}^{n \times \ell}$ and $\mat{b_\pi} \in \mathbb{R}^n, \mat{b_P} \in \mathbb{R}^{n^2}, \mat{b_r} \in \mathbb{R}^{n}$. Note the sizes of the matrices mean that every element of $\mat{\pi}, \mat{P}, \mat{r}$ must be bound with an affine equality to $\mat{\theta}$. In other words, if we fix $\mat{\theta}$, we can compute $\mat{\pi}, \mat{P}, \mat{r}$ exactly. The $\vect$ operator vectorizes a matrix, i.e., by concatenating rows into one column vector, so that $\vect(\mat{P}) \in \mathbb{R}^{n^2}$.

The affine equalities allow for common formulations like one of the parameters equaling an element of $\mat{\theta}$, e.g., $\pi_1 = \theta_1$. Another common situation is that we have a function that outputs the probability of transitioning to a specific state (say, a death state in a healthcare model), and the remaining probability mass is, e.g., distributed uniformly over the remaining states: $\pi_n = \theta_1$ and $\pi_i = (1 - \sum_{j \neq i} \theta_j) / (n-1)$ for $i \neq n$. Of course, the affine equalities also allow for parameters to be fixed as constants.

Next, we may have linear inequalities on the parameters: $\mat{C_\pi} \mat{\pi} \leq \mat{d_\pi}$, $\mat{C_P} \vect(\mat{P}) \leq \mat{d_P}$, and $\mat{C_r} \mat{r} \leq \mat{d_r}$, for fixed $\mat{C_\pi} \in \mathbb{R}^{k_{\mat{\pi}} \times n}, \mat{C_P} \in \mathbb{R}^{k_{\mat{P}} \times n^2}, \mat{C_r} \in \mathbb{R}^{k_{\mat{r}} \times n}$ and $\mat{d_\pi} \in \mathbb{R}^{k_{\mat{\pi}}}, \mat{d_P} \in \mathbb{R}^{k_{\mat{P}}}, \mat{d_r} \in \mathbb{R}^{k_{\mat{r}}}$, where $k_{\mat{\pi}}, k_{\mat{P}}, k_{\mat{r}}$ are the number of linear inequalities on $\mat{\pi}$, $\mat{P}$, and $\mat{r}$, respectively. 

Linear inequalities are an appropriately general assumption as they arise naturally from fixing a monotonic ordering on the rewards or the probabilities. In reliability engineering and healthcare, a common assumption is that of \emph{increasing failure rate}, which implies a certain stochastic ordering on the transition probabilities \cite{barlow1996mathematical, alagoz2007determining}. Effectively, this introduces a series of linear inequalities on the transition probabilities. See Figure~\ref{fig:maindiagram} for a diagram of the full pipeline.

Lastly, we assume that the feature vector $\mat{x}$ lies in a set $\mathcal{X} \subseteq \mathbb{R}^m$. We assume that the set $\mathcal{X}$ is \emph{mixed-integer-linear representable} (MILP-representable), i.e., it can be written as a finite union of polyhedra \cite{jeroslow1984integer}. This is satisfied by a wide range of sets, including box constraints, polytopes, and discrete sets. In practice, this may look like, e.g., that we restrict $\mat{x}$ to represent the set of patients that are either under 18 or over 65, as this is a union of two intervals.

\begin{figure*}
    \begin{center}
    \includegraphics[width=\textwidth]{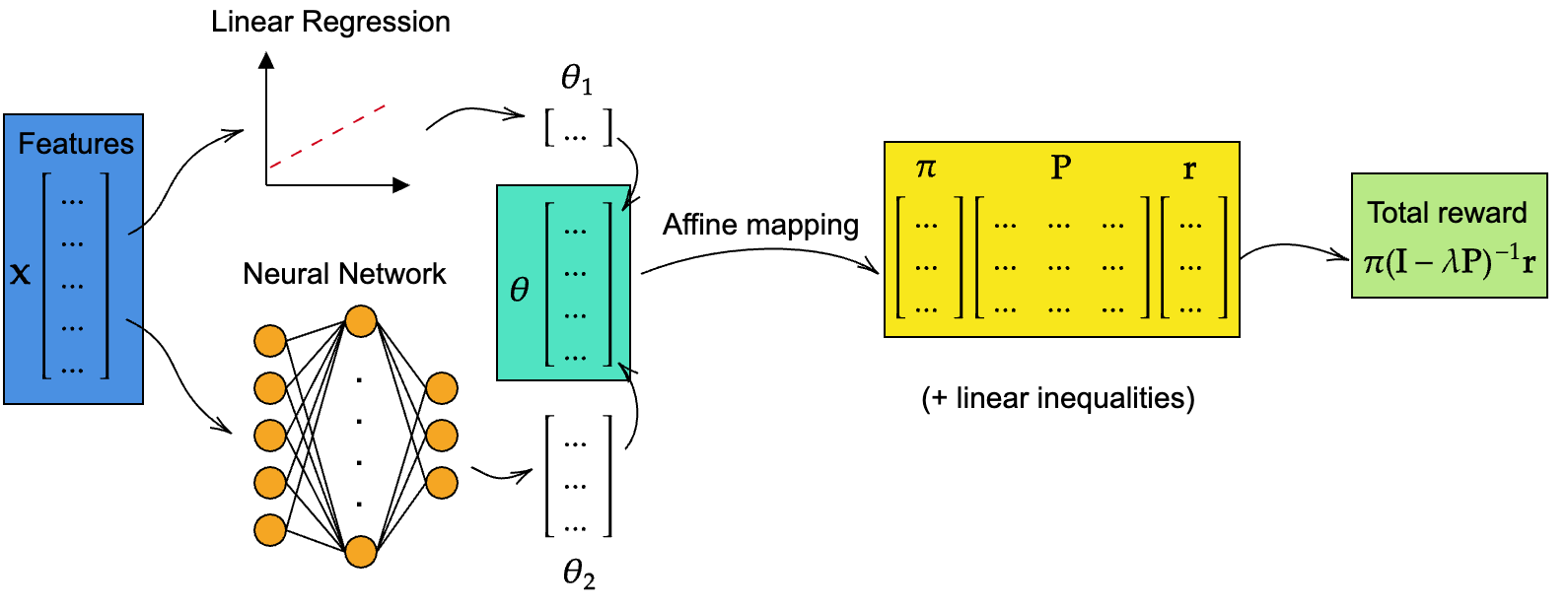}
    \caption{Example of our pipeline. A feature vector $\mat{x}$ is passed through different functions, here a linear regression and a neural network, to obtain the output vector $\mat{\theta}$, which then determines the parameters of the Markov process through affine equalities.}
    \label{fig:maindiagram}
    \end{center}
    \vskip -0.2in
\end{figure*}

With these assumptions, we can now formulate the following optimization problem for the total infinite-horizon discounted reward:
\begin{align*}
    &\minmax \limits_{\mat{\pi}, \mat{P}, \mat{r}, \mat{v}, \mat{x}, \mat{\theta}} \mat{\pi}^\top \mat{v} \quad \text{s.t.} \\
    &\quad \mat{v} = \lambda \mat{P} \mat{v} + \mat{r}, \\
    &\quad \mat{\theta} = \begin{bmatrix} \mat{\theta}_1, \mat{\theta}_2, \cdots, \mat{\theta}_{k_f} \end{bmatrix}^\top, \\
    &\quad \mat{\theta}_i = f_i(\mat{x}) \quad \forall i \in [k_f], \\
    &\quad \mat{\pi} = \mat{A_\pi} \mat{\theta} + \mat{b_\pi}, \\
    &\quad \vect(\mat{P}) = \mat{A_P} \mat{\theta} + \mat{b_P}, \\
    &\quad \mat{r} = \mat{A_r} \mat{\theta} + \mat{b_r}, \\
    &\quad \mat{C_\pi} \mat{\pi} \leq \mat{d_\pi}, \\
    &\quad \mat{C_P} \vect(\mat{P}) \leq \mat{d_P}, \\
    &\quad \mat{C_r} \mat{r} \leq \mat{d_r}, \\
    &\quad \sum_{j=1}^n P_{ij} = 1 \quad \forall i \in [n], \\
    &\quad \sum_{i=1}^n \pi_i = 1, \\
    &\quad 0 \leq P_{ij} \leq 1 \quad \forall i,j \in [n] \times [n], \\
    &\quad 0 \leq \pi_i \leq 1 \quad \forall i \in [n], \\
    &\quad \mat{x} \in \mathcal{X}.
\end{align*}
Above, $\mat{v}$ is a new vector we have introduced, which in Markov process theory would be referred to as the \emph{value function}, and often arises in reinforcement learning, which equals $\mat{v} = (\mat{I} - \lambda \mat{P})^{-1} \mat{r}$. We can easily convert this into a feasibility problem, to find a feasible $\mat{x}$, by replacing the objective function with an inequality, $W_{\min} \leq \mat{\pi}^\top \mat{v} \leq W_{\max}$. With slight modifications, we can also formulate reachability and hitting time: the full formulations are in Appendix~\ref{sec:app:other_formulations}.

\section{Solving the Optimization Problem}
\label{sec:solving}

Under the assumptions that the functions $f_1, f_2, \ldots, f_{k_f}$ are MILP-representable and $\mathcal{X}$ is MILP-representable, the optimization problem above is a \emph{mixed-integer bilinear program}, as we have a bilinear objective and a bilinear constraint. We will leverage the observation that strong bounds on variables are crucial for solving bilinear programs, and will derive such bounds by decomposing our problem and using various results from Markov theory.

Our strategy is as follows: obtain bounds on the elements of $\mat{\theta}$ by solving a series of smaller MILPs, propagate these bounds to the parameters via the affine equalities, obtain bounds on $\mat{v}$ using linear algebra arguments, and then tighten the $\mat{v}$ bounds using interval matrix analysis. All proofs are in Appendix \ref{sec:app:proofs_solving}.

\paragraph{Bounds on $\mat{\theta}$}
For each $\theta_i$, we minimize and maximize it by solving two smaller MILPs, over the set $\mathcal{X}$. Namely, for each $i \in [k_f]$, for each $j \in [\ell_i]$, solve for $\minmax_{\mat{x} \in \mathcal{X}} \theta_{i,j} \; \text{s.t.} \; \mat{\theta_i} = [\theta_{i,1}, \ldots, \theta_{i, \ell_i}] = f_i(\mat{x})$, which are MILPs as $\mathcal{X}$ and $f_1, \ldots f_{k_f}$ are MILP-representable.

\paragraph{Bounds on $\mat{\pi}, \mat{P}, \mat{r}$}
We now propagate the bounds on $\mat{\theta}$ to the parameters $\mat{\pi}, \mat{P}, \mat{r}$ via the affine equalities.

\begin{lemma}
    \label{lem:affine_bounds}
    Let $\mat{\theta} \in \mathbb{R}^\ell$ be a vector where each component satisfies the bounds $\theta_i^{\min} \leq \theta_i \leq \theta_i^{\max}$ for all $i = 1, 2, \ldots, \ell$. Consider an affine transformation defined by $\mat{y} = \mat{A} \mat{\theta} + \mat{b}$, where $\mat{y} \in \mathbb{R}^k$, $\mat{A} \in \mathbb{R}^{k \times \ell}$, and $\mat{b} \in \mathbb{R}^k$. Then, each component $y_i$ of $\mat{y}$ satisfies $b_i + \sum_{j=1}^\ell A_{ij} \cdot \left(\theta_j^{\min} \ind_{A_{ij} \geq 0} + \theta_j^{\max} \ind_{A_{ij} < 0}\right) \leq y_i \leq b_i + \sum_{j=1}^\ell A_{ij} \cdot \left(\theta_j^{\max} \ind_{A_{ij} \geq 0} + \theta_j^{\min} \ind_{A_{ij} < 0}\right)$.
\end{lemma}

\paragraph{Bounds on $\mat{v}$}
Next, we derive bounds on $\mat{v}$ using well-established facts from Markov theory. First, as $\mat{v} = (\mat{I} - \lambda \mat{P})^{-1} \mat{r}$, we analyze bounds on $(\mat{I} - \lambda \mat{P})^{-1}$.
\begin{lemma}
    \label{lem:inverse_bounds}
    Let $\mat{P}$ be row-stochastic and $\lambda \in (0,1)$. Then, each element in $(\mat{I} - \lambda \mat{P})^{-1}$ is in the interval $[0, 1/(1-\lambda)]$ and the row sums are all equal to $1/(1-\lambda)$.
\end{lemma}

Now we can propagate the above bounds to $\mat{v}$.

\begin{lemma}
    \label{lem:v_bounds}
    Let $r_i^{\min} \leq r_i \leq r_i^{\max}$. Then, each element in the vector $\mat{v}$ satisfies $\gamma \min_j r_j^{\min} \leq v_i \leq \gamma \max_j r_j^{\max}$, where $\gamma = \frac{1}{1-\lambda}$.
\end{lemma}

Lemma~\ref{lem:v_bounds} provides valid bounds on $\mat{v}$. However, these bounds are not tight, as Lemma~\ref{lem:inverse_bounds} applies to \emph{any} row-stochastic matrix, and does not take into account the specific element-wise bounds on $\mat{P}$ that we derived from Lemma~\ref{lem:affine_bounds}. In order to derive tighter bounds on $\mat{v}$, we need to incorporate the specific element-wise bounds on $\mat{P}$. 
This is difficult because of the matrix inverse. In the following section, we show how to solve this problem with interval matrix analysis and hence derive tighter bounds on $\mat{v}$.

\subsection{Tightening bounds on $\mat{v}$}
\label{subsec:tightening_bounds}

Interval matrix analysis studies a generalization of matrices (or vectors), where each element is an \emph{interval} instead of a real number (see \citep{neumaier1984linear} for an overview). Equivalently, it can be viewed as studying the set of all matrices that lie within given intervals. There are easy generalizations of arithmetic of real numbers to intervals, which can be used to generalize matrix arithmetic to interval matrices. However, inverting an interval matrix is difficult due to the inherent nonlinearity of the inverse operation. Consider the equivalent problem of solving a linear system with interval matrices. The solution set is not rectangular, so cannot be described as an interval matrix. Hence, we wish to find the interval matrix of smallest radius that contains the solution set, known as the \emph{hull} \cite{neumaier1984linear}.

Due to the interval extensions of the basic arithmetic operations, it is possible to generalize the Gauss-Seidel method, a common iterative procedures for solving linear systems, to interval matrices. Beginning with an initial enclosure of the solution set, it iteratively refines the intervals and returns a smaller enclosure of the solution set. We provide a full description of the Gauss-Seidel method in Appendix~\ref{sec:app:gauss_seidel}. The interval Gauss-Seidel method is optimal in the sense that it provides the smallest enclosure of the solution set out of a broad class of algorithms \cite{neumaier1984linear}.

Hence, we can apply the Gauss-Seidel method to tighten the bounds on $\mat{v}$, as $(\mat{I} - \lambda \mat{P}) \mat{v} = \mat{r}$ is a linear system. We use the bounds on $\mat{P}$ and $\mat{r}$ from Lemma~\ref{lem:affine_bounds} to form the interval matrices, and for our initial guess, we can use the bounds derived from Lemma~\ref{lem:v_bounds}.

Still, Gauss-Seidel is not guaranteed to obtain the hull itself. A well-known sufficient condition for Gauss-Seidel to obtain the hull is that the matrix in the linear system is an \emph{interval M-matrix}, a generalization of the notion of an M-matrix \cite{berman1994nonnegative} to interval matrices. Suppose a matrix $\mat{M}$ is element-wise bounded by $\mat{M^{\max}}$ and $\mat{M^{\min}}$. It is an \emph{interval M-matrix} if and only if $\mat{M^{\min}}$ is a (non-singular) real M-matrix in the usual sense and $\mat{M^{\max}}$ has no positive off-diagonal elements \cite{neumaier1984linear}. Hence we can  prove a necessary and sufficient condition for $\mat{I} - \lambda \mat{P}$ to be an interval M-matrix.

\begin{theorem}
    \label{thm:interval_m_matrix}
    Let $\mat{P}$ be a row-stochastic matrix bounded element-wise by $\mat{P^{\min}} \leq \mat{P} \leq \mat{P^{\max}}$. Then, $\mat{I} - \lambda \mat{P}$ is an interval M-matrix if and only if $\rho(\mat{P}^{\max}) \leq \frac{1}{\lambda}$, where $\rho$ is the spectral radius.
\end{theorem}

Notably, we only need check the spectral radius of the matrix formed by the upper bounds on $\mat{P}$. There are various conditions that can be used to obtain the necessary bound on the spectral radius. Since $\lambda < 1$, we have $1/\lambda > 1$, so if we can show that the spectral radius of the upper bound matrix is less than $1$, then this automatically implies the condition in Theorem~\ref{thm:interval_m_matrix}, and hence that the Gauss-Seidel method will obtain the hull, i.e., the tightest possible rectangular bounds on $\mat{v}$. This may occur if, for example, the row sums of the upper bounds are $\leq 1$.

\begin{example}
    Let $\lambda=0.97$, and suppose we have the following bounds on $\mat{P}$ and $\mat{r}$:
    \begin{align*}
        \mat{P}:
        \begin{bmatrix}
            [0.5, 0.6] & [0.2, 0.5] \\
            [0.1, 0.4] & [0.5, 0.6]
        \end{bmatrix},\quad
        \mat{r}:
        \begin{bmatrix}
            [0, 100] \\
            [0, 100]
        \end{bmatrix}
    \end{align*}
    The bounds on $\mat{v}$ are:
    \begin{align*}
        \begin{array}{c@{\hskip 1cm}c}
            \text{Initial (Lemma~\ref{lem:v_bounds})} & \text{After Gauss-Seidel} \\
            \begin{bmatrix}
                [0, 6666.67] \\
                [0, 6666.67]
            \end{bmatrix} 
            &
            \begin{bmatrix}
                [39.36, 6666.67] \\
                [104.50, 6666.67]
            \end{bmatrix}
        \end{array}
    \end{align*}
    Note that the spectral radius of the upper bound matrix of $\mat{P}$ is $1.05$, which is greater than $1/\lambda = 1/0.97 \approx 1.03$. So, Gauss-Seidel will not necessarily obtain the hull, as it is not an interval M-matrix, by Theorem~\ref{thm:interval_m_matrix}.

    If we slightly tighten the upper bounds on $\mat{P}$ by changing the upper bound on $P_{2,1}$ to $0.3$, then the bounds on $\mat{v}$ are:
    \begin{align*}
        \begin{array}{c@{\hskip 1cm}c}
            \text{Initial (Lemma~\ref{lem:v_bounds})} & \text{After Gauss-Seidel} \\
            \begin{bmatrix}
                [0, 6666.67] \\
                [0, 6666.67]
            \end{bmatrix} 
            &
            \begin{bmatrix}
                [39.37, 4694.80] \\
                [104.50, 3746.86]
        \end{bmatrix}
        \end{array}
    \end{align*}
    Indeed, now the spectral radius of the upper bound matrix of $\mat{P}$ is $0.99$, which is less than $1/\lambda = 1.03$, so by Theorem~\ref{thm:interval_m_matrix}, Gauss-Seidel returns the hull. Note the tighter bounds than the prior case.
\end{example}

\subsection{Solving the final optimization problem}
At last, we can assert the bounds found above for $\mat{\pi}, \mat{P}, \mat{r}, \mat{v}$ as constraints and solve the full bilinear program with any out-of-the-box solver. We summarize our procedure in Algorithm~\ref{alg:solving}.

\begin{algorithm}
    \caption{Decomposition and bound propagation scheme to optimize the total reward}
    \label{alg:solving}
    \begin{algorithmic}
        \STATE {\bfseries Input:} Instance of the infinite-horizon total reward optimization problem
        \FOR{$i=1$ {\bfseries to} $k_f$, {$j=1$ {\bfseries to} $\ell_i$}}
            \STATE Minimize and maximize $\theta_{i,j}$ over $\mathcal{X}$
        \ENDFOR
        \STATE Propagate the bounds on $\mat{\theta}$ to $\mat{\pi}, \mat{P}, \mat{r}$ using Lemma~\ref{lem:affine_bounds}
        \STATE Obtain initial bounds on $\mat{v}$ using Lemma~\ref{lem:v_bounds}
        \STATE Tighten the bounds on $\mat{v}$ using Gauss-Seidel
        \STATE Optimize $\mat{\pi}^\top \mat{v}$ with an out-of-the-box solver, with obtained bounds on $\mat{\pi}, \mat{P}, \mat{r}, \mat{v}$ as constraints
    \end{algorithmic}
\end{algorithm}

Our algorithm works with any standard bilinear or MILP solver, avoiding the need for specialized branch-and-bound schemes or custom callbacks that require expert implementation. The bound propagation and Gauss-Seidel are also straightforward to implement. Crucially, our method solves the problem to global optimality if the underlying bilinear solver can, and so we do not discuss optimality guarantees further, as, e.g., Gurobi can solve bilinear problems to global optimality.

\paragraph{Extensions: reachability and hitting time}

Appendix ~\ref{sec:app:solving_other_formulations} shows that our results also extend to reachability and hitting time, after handling technical details arising from the invertibility of $\mat{I} - \mat{Q}$. 

\paragraph{Extensions: special cases}

If one or more of $\mat{\pi}, \mat{P}, \mat{r}$ are fixed, then the problem may become much easier to solve. For example, if $\mat{P}$ and $\mat{r}$ are fixed, then $\mat{v}$ is completely determined as the solution to a linear system, so we can eliminate the Bellman constraint, and the objective becomes simply linear in $\pi$. We enumerate all of these cases in Appendix~\ref{sec:app:special_cases}.

\section{Implementation}
\label{sec:implementation}

We developed the Python package \texttt{markovml} to specify Markov chains or reward processes with embedded pretrained machine learning models. Our domain-specific language lets users:
\begin{enumerate}[noitemsep,topsep=0pt]
    \item Instantiate a Markov process
    \item Add pretrained ML models from \texttt{sklearn} \cite{scikit-learn} and PyTorch \cite{paszke2019pytorch}
    \item Link model outputs to Markov parameters with affine equalities 
    \item Include extra linear inequalities
    \item Specify the feature set with MILP constraints
    \item Optimize reachability, hitting time, or total reward
\end{enumerate}

Built on \texttt{gurobipy} \cite{gurobi}, our package supports a variety of regression and classification models, including linear, tree-based, and neural networks. For some models, the MILP formulation is provided by \texttt{gurobi-machinelearning} \cite{gurobi_ml}, while for some models (e.g., softmax for classifiers), we implemented the MILP formulation ourselves. We also allow the user to build ``if-then" rules in a natural syntax. Our package fully implements Algorithm~\ref{alg:solving}. See Appendix~\ref{sec:app:implementation} for more details.

\section{Numerical experiments}
\label{sec:experiments}

We compare our method to solving the optimization problem directly with an out-of-the-box solver. Our experiments will analyze the following drivers of complexity of the bilinear program: the number of states, the number of ML models, and the model complexity, notably for trees and neural networks.

\paragraph{Setup}

For our experiments, we generate Markov chains and train ML models on randomly generated data. For rewards, we train regression models and for probabilities, we train classifier models, with parameters either linked to outputs or set to constants. For each configuration in our experiments, we generate 10 instances, and solve the optimization problem directly with an out-of-the-box solver (Gurobi) and with our method. We perform the following experiments: (1) varying the number of states, (2) varying the number of ML models, (3) varying decision tree depth, and (4) varying the architecture of neural networks. We record runtime, objective value, and optimizer status. For our method, runtime includes the full runtime, i.e., counting solving the smaller MILPs, propagating bounds, interval Gauss-Seidel, and the final optimization step. We describe the full setup and analysis in Appendix~\ref{sec:app:num_exp}. 

\begin{figure*}[!ht]
    \begin{center}
    \includegraphics[width=\textwidth]{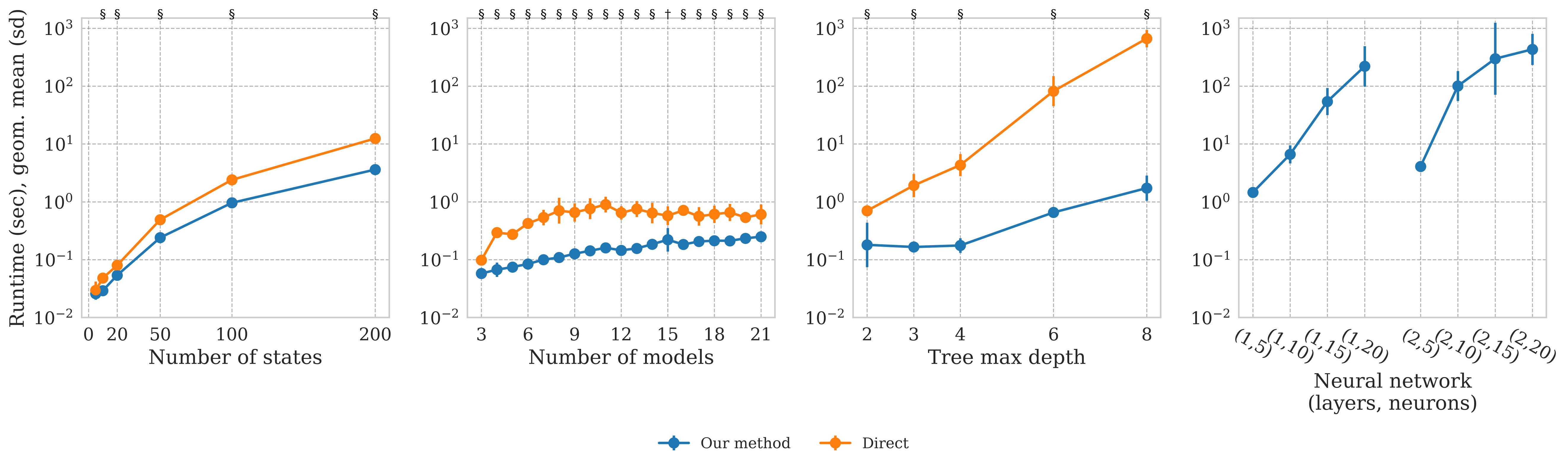}
    \caption{Runtimes of our method versus direct solving. Each panel shows results from experiments 1-4 (left to right). Points represent geometric means with error bars indicating standard deviation. Statistical significance from paired t-tests shown above: * ($p < 0.05$), † ($p < 0.01$), and § ($p < 0.001$).}
    \label{fig:resultsplot}
    \end{center}
    \vskip -0.2in
\end{figure*}

\begin{figure*}[!ht]
    \begin{center}
    \includegraphics[width=\textwidth]{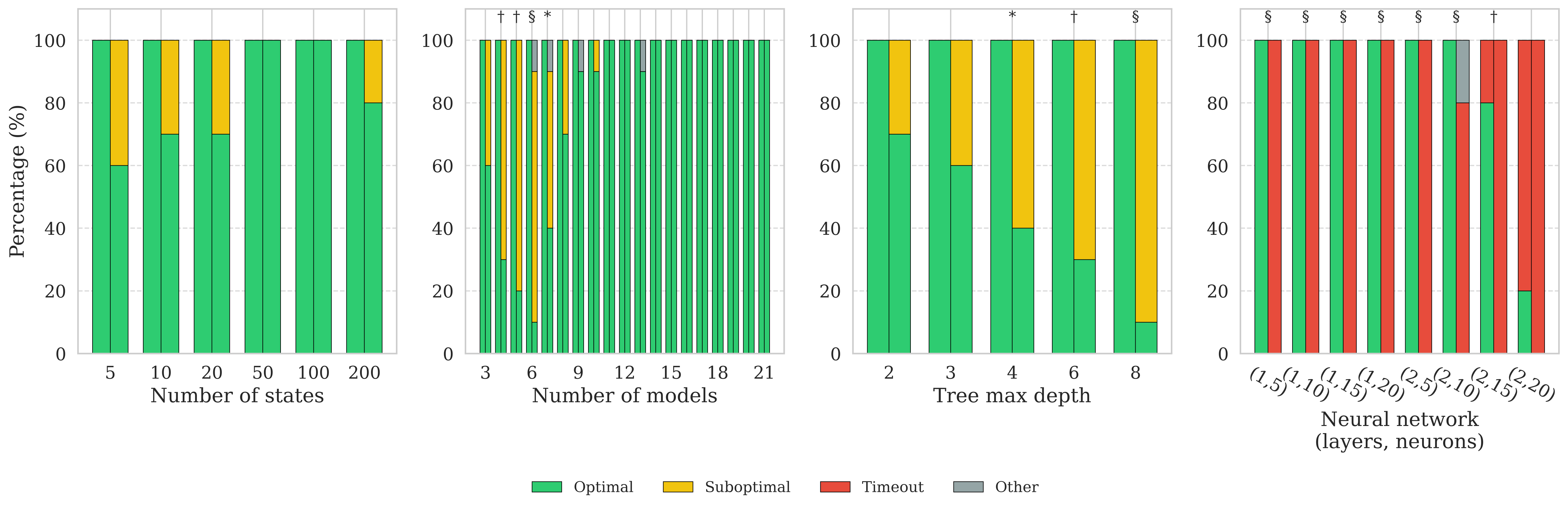}
    \caption{Proportion (\%) of instances categorized as optimal, suboptimal, timed out, or other. Each panel shows experiments 1-4 (left to right). Our method is the left bar and direct method is the right bar. Statistical significance from $\chi^2$ shown above: * ($p < 0.05$), † ($p < 0.01$), and § ($p < 0.001$).}
    \label{fig:resultspropsplot}
    \end{center}
    \vskip -0.2in
\end{figure*}

\subsection{Results}

Figure \ref{fig:resultsplot} shows that our method consistently outperforms direct solving and  scales much better with all measures of complexity, and the differences are statistically significant, often at 1\% and 0.1\% significance. This is despite our method involving multiple optimization problems.

The difference is particularly striking as models get more complex: in the tree depth experiments, our method is on average 117x faster for trees of depth $6$ (direct mean 82.1 sec, our method mean 0.7 sec), and 391x faster for depth $8$ (direct mean 665.3 sec, our method mean 1.7 sec) -- reaching a 1000x speedup on some instances. For neural networks, solving the problem directly times out in every instance, as seen in Figure \ref{fig:resultspropsplot}, even for the smallest architecture with 1 hidden layer of 5 neurons. However, our method solves all instances with 1 hidden layer and most instances with 2 hidden layers, although it often times out with 2 hidden layers of 20 neurons each. Figure \ref{fig:resultspropsplot} also shows that the direct method often struggles with proving optimality.

In order to isolate the effects of different parts of our method, we performed an ablation study. We redid the decision trees experiment fixing a tree depth of $8$. We ablated, in turn: (1) finding bounds on $\mat{\theta}$, (2) propagating bounds to $\mat{\pi}, \mat{P}, \mat{r}$, (3) initial bounds on $\mat{v}$, and (4) tightening bounds on $\mat{v}$. Ablating a step means ablating it and all subsequent steps. As shown in Table \ref{tab:ablation}, the biggest contributor to our method's speedup is obtaining the initial bounds on $\mat{v}$, which causes a 153x speedup on average compared to ablating it. 

\begin{table}[t]
\caption{Geometric means of runtimes and number of instances solved for ablations of our method. Ablation numbers refer to ablating: (1) finding bounds on $\mat{\theta}$, (2) propagating bounds to $\mat{\pi}, \mat{P}, \mat{r}$, (3) initial bounds on $\mat{v}$, and (4) tightening bounds on $\mat{v}$. Lastly, None refers to no ablations. Note that ablating (1) is equivalent to directly solving, and None is equivalent to our full method.}
\label{tab:ablation}
\vskip -0.2in
\begin{center}
\begin{tiny}
\begin{tabular}{cccccc}
\toprule
& (1) & (2) & (3) & (4) & None \\
\midrule
\makecell{Runtime (sec), \\ geom. mean (sd)} &
\makecell{710.79 \\ (1.22)} & 
\makecell{292.36 \\ (1.64)} & 
\makecell{282.42 \\ (1.59)} & 
\makecell{1.85 \\ (1.55)} & 
\makecell{1.89 \\ (1.56)} \\
\midrule
Instances solved & 6/10 & 8/10 & 7/10 & 10/10 & 10/10 \\
\bottomrule
\end{tabular}
\end{tiny}
\end{center}
\vskip -0.3in
\end{table}

\section{Case Study}

As a short case study, we re-analyze the healthcare cost-benefit analysis of \citet{maaz2025cost}, which models patients after an out-of-hospital cardiac arrest (OHCA). Patients are discharged with a \emph{modified Rankin scale (mRS)}, a measure of post-arrest morbidity (0: perfect health, 6: death), which determines yearly mortality and quality of life. The Markov chain in their model has 20 states in total: discharged with mRS 0-6, first-year post-discharge with mRS 0-5, later-year post-discharge with mRS 0-5, and death. The only allowable transitions are, for each mRS, discharge to first-year, first-year to later-year, later-year to later-year; and, any state to death. A patient's discharge mRS, i.e., $\mat{\pi}$, is given by an ML model that takes as input the patient's age and sex, features of the OHCA, and the time-to-defibrillation. $\mat{P}$ and $\mat{r}$, representing the \textit{net monetary benefit} (NMB), a common health economic metric, are fixed and taken from literature.

\citet{maaz2025cost} perform their analysis on data from 22,017 real OHCAs in Canada. We obtained the data they used, trained a decision tree, and reconstructed their Markov process in our software package \texttt{markovml}. By setting appropriate constraints to build $\mathcal{X}$, we can now perform subgroup analyses. For example, we found that the best-case (or, maximum) NMB for men is about 50\% higher than for women. This reflects lower mortality among men, i.e., lower probability of being discharged with mRS 6 (death). Indeed, this is corroborated by the healthcare literature which finds that women have at least 50\% higher mortality than men after an OHCA \cite{parikh2020association}. Notably, doing this analysis as an optimization problem allowed us to solve it in an automated way, without requiring the underlying patient data nor simulating representative patients. We refer the reader to the code attached with our paper, which contains the ML model (due to data privacy, we cannot share the patient data), as well as further details on the implementation. 

\section{Discussion and Conclusion}

We have introduced a novel framework for the formal verification of Markov processes with learned parameters, enabling exact reasoning about systems where transition probabilities and rewards are specified by machine learning models. Our approach formulates the verification task as a bilinear program and introduces a decomposition and bound propagation scheme that significantly accelerates computation—often by orders of magnitude, while maintaining global optimality guarantees. This work contributes to the broader goal of safe and transparent deployment of ML in high-stakes settings. Our healthcare case study demonstrates how the method can support robust decision making in critical domains. We have implemented this in the software package \texttt{markovml}, which provides a flexible language for defining Markov processes, embedding a range of ML models, and optimizing key metrics.

\paragraph{Limitations}
Our experiments show that the method scales well across problem dimensions including the number of states, models, and model complexity. However, a key limitation is scalability for neural networks: our method worked well on small networks but degrades with increasing depth or width. This is consistent with known computational hardness results for neural network verification. Future work could improve performance by integrating state-of-the-art techniques from the neural network verification literature, such as $\alpha, \beta$-CROWN.

\paragraph{Future directions}
Our software package, \texttt{markovml}, is open-source and designed for extensibility. We envision future extensions that support richer ML models and eventually integrates with other neural network verifiers. We hope this work encourages further research at the intersection of probabilistic model checking and formal verification of ML.

\newpage

\bibliography{ref}
\bibliographystyle{unsrtnat}

\newpage
\appendix

\section{Formulations of feasibility, reachability, and hitting time}
\label{sec:app:other_formulations}

\subsection{Feasibility of the total infinite-horizon reward}

The total infinite-horizon reward feasibility problem can be written as:

\begin{align*}
    &\exists \mat{x} \quad \text{s.t.} \\
    &\quad W_{\min} \leq \mat{\pi}^\top \mat{v} \leq W_{\max} \\
    &\quad \mat{v} = \lambda \mat{P} \mat{v} + \mat{r}, \\
    &\quad \mat{\theta} = \begin{bmatrix} \mat{\theta}_1, \mat{\theta}_2, \cdots, \mat{\theta}_{k_f} \end{bmatrix}^\top, \\
    &\quad \mat{\theta}_i = f_i(\mat{x}) \quad \forall i \in [k_f], \\
    &\quad \mat{\pi} = \mat{A_\pi} \mat{\theta} + \mat{b_\pi}, \\
    &\quad \vect(\mat{P}) = \mat{A_P} \mat{\theta} + \mat{b_P}, \\
    &\quad \mat{r} = \mat{A_r} \mat{\theta} + \mat{b_r}, \\
    &\quad \mat{C_\pi} \mat{\pi} \leq \mat{d_\pi}, \\
    &\quad \mat{C_P} \vect(\mat{P}) \leq \mat{d_P}, \\
    &\quad \mat{C_r} \mat{r} \leq \mat{d_r}, \\
    &\quad \sum_{j=1}^n P_{ij} = 1 \quad \forall i \in [n], \\
    &\quad \sum_{i=1}^n \pi_i = 1, \\
    &\quad 0 \leq P_{ij} \leq 1 \quad \forall i,j \in [n] \times [n], \\
    &\quad 0 \leq \pi_i \leq 1 \quad \forall i \in [n], \\
    &\quad \mat{x} \in \mathcal{X}.
\end{align*}

where $W_{\min}$ and $W_{\max}$ are bounds on the total reward, which may be $\pm \infty$.

\subsection{Reachability}

To reformulate the total reward formulation for reachability to a target state $S \subseteq [n]$ from a set of transient states $T \subseteq [n], T \cap S = \emptyset$, observe that we can simply make the following substitutions: $\mat{\pi}$ gets replaced with $\mat{\tilde{\pi}}$, which is the initial probability over $T$, $\mat{P}$ gets replaced with $\mat{Q}$, which is $\mat{P}$ restricted to $T$, and $\mat{r}$ gets replaced with $\mat{R} \mat{1}$, where $\mat{R}$ is the transition matrix from $T$ to $S$. Lastly, there is no more discount rate $\lambda$. We also require $\mat{Q}$ to be strictly substochastic so that it remains invertible (i.e., the row sums are strictly less than $1$), and $\mat{R}$ and $\mat{\tilde{\pi}}$ are now substochastic (i.e., the row sums are less than or equal to $1$). The optimization problem is then:

\begin{align*}
    &\minmax \limits_{\mat{\tilde{\pi}}, \mat{Q}, \mat{R}, \mat{v}, \mat{x}, \mat{\theta}} \mat{\tilde{\pi}}^\top \mat{v} \quad \text{s.t.} \\
    &\quad \mat{v} = \mat{Q} \mat{v} + \mat{R} \mat{1}, \\
    &\quad \mat{\theta} = \begin{bmatrix} \mat{\theta}_1, \mat{\theta}_2, \cdots, \mat{\theta}_{k_f} \end{bmatrix}^\top, \\
    &\quad \mat{\theta}_i = f_i(\mat{x}) \quad \forall i \in [k_f], \\
    &\quad \mat{\pi} = \mat{A_\pi} \mat{\theta} + \mat{b_\pi}, \\
    &\quad \vect(\mat{Q}) = \mat{A_Q} \mat{\theta} + \mat{b_Q}, \\
    &\quad \vect(\mat{R}) = \mat{A_R} \mat{\theta} + \mat{b_R}, \\
    &\quad \mat{C_{\tilde{\pi}}} \mat{\tilde{\pi}} \leq \mat{d_{\tilde{\pi}}}, \\
    &\quad \mat{C_{QR}} [\vect(\mat{Q}) \vect(\mat{R})] \leq \mat{d_{QR}}, \\
    &\quad \sum_{j=1}^{|T|} Q_{ij} < 1 \quad \forall i \in [|T|], \\
    &\quad \sum_{j=1}^{|S|} R_{ij} \leq 1 \quad \forall i \in [|T|], \\
    &\quad \sum_{i=1}^{|T|} (\tilde{\pi})_i \leq 1, \\
    &\quad 0 \leq Q_{ij} < 1 \quad \forall i,j \in [|T|] \times [|T|], \\
    &\quad 0 \leq R_{ij} \leq 1 \quad \forall i,j \in [|T|] \times [|S|], \\
    &\quad 0 \leq (\tilde{\pi})_i \leq 1 \quad \forall i \in [|T|], \\
    &\quad \mat{x} \in \mathcal{X}.
\end{align*}

With a slight abuse of notation, we still maintain the vector $\mat{v}$, but is instead now defined as $\mat{v} = (\mat{I} - \mat{Q})^{-1} \mat{R} \mat{1}$. As well, note that the stochastic constraints have been replaced with substochastic constraints (strict substochastic for $\mat{Q}$).

The feasibility problem is:

\begin{align*}
    &\exists \mat{x} \quad \text{s.t.} \\
    &\quad W_{\min} \leq \mat{\tilde{\pi}}^\top \mat{v} \leq W_{\max} \\
    &\quad \mat{v} = \mat{Q} \mat{v} + \mat{R} \mat{1}, \\
    &\quad \mat{\theta} = \begin{bmatrix} \mat{\theta}_1, \mat{\theta}_2, \cdots, \mat{\theta}_{k_f} \end{bmatrix}^\top, \\
    &\quad \mat{\theta}_i = f_i(\mat{x}) \quad \forall i \in [k_f], \\
    &\quad \mat{\tilde{\pi}} = \mat{A_{\tilde{\pi}}} \mat{\theta} + \mat{b_{\tilde{\pi}}}, \\
    &\quad \vect(\mat{Q}) = \mat{A_Q} \mat{\theta} + \mat{b_Q}, \\
    &\quad \vect(\mat{R}) = \mat{A_R} \mat{\theta} + \mat{b_R}, \\
    &\quad \mat{C_{\tilde{\pi}}} \mat{\tilde{\pi}} \leq \mat{d_{\tilde{\pi}}}, \\
    &\quad \mat{C_{QR}} [\vect(\mat{Q}) \vect(\mat{R})] \leq \mat{d_{QR}}, \\
    &\quad \sum_{j=1}^{|T|} Q_{ij} < 1 \quad \forall i \in [|T|], \\
    &\quad \sum_{j=1}^{|S|} R_{ij} \leq 1 \quad \forall i \in [|T|], \\
    &\quad \sum_{i=1}^{|T|} \tilde{\pi}_i \leq 1, \\
    &\quad 0 \leq Q_{ij} < 1 \quad \forall i,j \in [|T|] \times [|T|], \\
    &\quad 0 \leq R_{ij} \leq 1 \quad \forall i,j \in [|T|] \times [|S|], \\
    &\quad 0 \leq \tilde{\pi}_i \leq 1 \quad \forall i \in [|T|], \\
    &\quad \mat{x} \in \mathcal{X}.
\end{align*}

\subsection{Hitting time}

The hitting time to a target set $S \subseteq [n]$ from a set of transient states $T \subseteq [n], T \cap S = \emptyset$ formulation is extremely similar to the reachability formulation, except that we replace $\mat{R} \mat{1}$ with $\mat{1}$. Hence, $\mat{R}$ is no longer a decision variable. Also, now, $\mat{v}$ is defined as $\mat{v} = (\mat{I} - \mat{Q})^{-1} \mat{1}$. Note again the substochastic constraint on $\mat{\tilde{\pi}}$ and the strict substochastic constraint on $\mat{Q}$. The optimization problem is:

\begin{align*}
    &\minmax \limits_{\mat{\tilde{\pi}}, \mat{Q}, \mat{v}, \mat{x}, \mat{\theta}} \mat{\tilde{\pi}}^\top \mat{v} \quad \text{s.t.} \\
    &\quad \mat{v} = \mat{Q} \mat{v} + \mat{1}, \\
    &\quad \mat{\theta} = \begin{bmatrix} \mat{\theta}_1, \mat{\theta}_2, \cdots, \mat{\theta}_{k_f} \end{bmatrix}^\top, \\
    &\quad \mat{\theta}_i = f_i(\mat{x}) \quad \forall i \in [k_f], \\
    &\quad \mat{\tilde{\pi}} = \mat{A_{\tilde{\pi}}} \mat{\theta} + \mat{b_{\tilde{\pi}}}, \\
    &\quad \vect(\mat{Q}) = \mat{A_Q} \mat{\theta} + \mat{b_Q}, \\
    &\quad \mat{C_{\tilde{\pi}}} \mat{\tilde{\pi}} \leq \mat{d_{\tilde{\pi}}}, \\
    &\quad \mat{C_{Q}} \vect(\mat{Q}) \leq \mat{d_{Q}}, \\
    &\quad \sum_{j=1}^{|T|} Q_{ij} < 1 \quad \forall i \in [|T|], \\
    &\quad \sum_{i=1}^{|T|} \tilde{\pi}_i \leq 1, \\
    &\quad 0 \leq Q_{ij} < 1 \quad \forall i,j \in [|T|] \times [|T|], \\
    &\quad 0 \leq \tilde{\pi}_i \leq 1 \quad \forall i \in [|T|], \\
    &\quad \mat{x} \in \mathcal{X}.
\end{align*}

The feasibility problem is:

\begin{align*}
    &\exists \mat{x} \quad \text{s.t.} \\
    &\quad W_{\min} \leq \mat{\tilde{\pi}}^\top \mat{v} \leq W_{\max} \\
    &\quad \mat{v} = \mat{Q} \mat{v} + \mat{1}, \\
    &\quad \mat{\theta} = \begin{bmatrix} \mat{\theta}_1, \mat{\theta}_2, \cdots, \mat{\theta}_{k_f} \end{bmatrix}^\top, \\
    &\quad \mat{\theta}_i = f_i(\mat{x}) \quad \forall i \in [k_f], \\
    &\quad \mat{\tilde{\pi}} = \mat{A_{\tilde{\pi}}} \mat{\theta} + \mat{b_{\tilde{\pi}}}, \\
    &\quad \vect(\mat{Q}) = \mat{A_Q} \mat{\theta} + \mat{b_Q}, \\
    &\quad \mat{C_{\tilde{\pi}}} \mat{\tilde{\pi}} \leq \mat{d_{\tilde{\pi}}}, \\
    &\quad \mat{C_{Q}} \vect(\mat{Q}) \leq \mat{d_{Q}}, \\
    &\quad \sum_{j=1}^{|T|} Q_{ij} < 1 \quad \forall i \in [|T|], \\
    &\quad \sum_{i=1}^{|T|} (\pi_T)_i \leq 1, \\
    &\quad 0 \leq Q_{ij} < 1 \quad \forall i,j \in [|T|] \times [|T|], \\
    &\quad 0 \leq \tilde{\pi}_i \leq 1 \quad \forall i \in [|T|], \\
    &\quad \mat{x} \in \mathcal{X}.
\end{align*}

\section{Omitted Proofs from Section~\ref{sec:solving}}
\label{sec:app:proofs_solving}

\begin{proof}[Proof of Lemma~\ref{lem:affine_bounds}]
    For a given component $y_i$, we can write the transformation as:
    \begin{align*}
        y_i = \sum_{j=1}^{\ell} A_{ij} \theta_j + b_i
    \end{align*}
    for all $i = 1, 2, \ldots, k$.

    Each term in the summation depends linearly on $\theta_j$. If $A_{ij} \geq 0$, then $A_{ij} \theta_j$ is maximized at $\theta_j^{\max}$ and minimized at $\theta_j^{\min}$. Vice versa, if $A_{ij} < 0$, then $A_{ij} \theta_j$ is maximized at $\theta_j^{\min}$ and minimized at $\theta_j^{\max}$.

    Applying this to all terms in the summation, we get the required bounds.
\end{proof}

\begin{proof}[Proof of Lemma~\ref{lem:inverse_bounds}]
    Note that the matrix $(\mat{I} - \lambda \mat{P})$ is a (non-singular) M-matrix. By Theorem 2.3 in Chapter 6 of \cite{berman1994nonnegative}, each element of the matrix is non-negative, so is $\geq 0$.

    Next, consider the row sum of $(\mat{I} - \lambda \mat{P})^{-1}$ for a row $i$: $\sum_j (\mat{I} - \lambda \mat{P})^{-1}_{ij}$, and then expand into the Neumann series: $\sum_j \sum_{k=0}^\infty \lambda^k (\mat{P}^k)_{ij}$. We can exchange the order of summation as all terms are non-negative, and then we have our sum is $\sum_{k=0}^\infty \lambda^k \sum_j (\mat{P}^k)_{ij}$. Every positive integer exponent of a row-stochastic matrix is also row-stochastic, so $\sum_j (\mat{P}^k)_{ij} = 1$, for all $k \in \mathbb{N}$. Therefore, this simplifes to the geometric series $\sum_{k=0}^\infty \lambda^k$, which equals $\frac{1}{1-\lambda}$. As each row sum is $\frac{1}{1-\lambda}$, and each element is non-negative, each element must also be at most $\frac{1}{1-\lambda}$.
\end{proof}

\begin{proof}[Proof of Lemma~\ref{lem:v_bounds}]
    Recall that $\mat{v} = (\mat{I} - \lambda \mat{P})^{-1} \mat{r}$, so $v_i = \sum_{j=1}^n (\mat{I} - \lambda \mat{P})^{-1}_{ij} r_j$. Thus, we can see $v_i$ is a positive linear combination of the elements of $\mat{r}$, where the nonnegative weights are from $(\mat{I} - \lambda \mat{P})^{-1}$ and sum to $1/(1-\lambda)$. To maximize this combination, we assign all the weight to the maximum upper bound of $\mat{r}$, and to minimize it, we assign all the weight to the minimum lower bound of $\mat{r}$. Hence, we have the required bounds.
\end{proof}

\begin{proof}[Proof of Theorem~\ref{thm:interval_m_matrix}]
    Let $\mat{M} = \mat{I} - \lambda \mat{P}$. Clearly, $\mat{M}$ is bounded element-wise by $\mat{M^{\max}} = \mat{I} - \lambda \mat{P^{\min}}$ and $\mat{M^{\min}} = \mat{I} - \lambda \mat{P^{\max}}$. The upper bound $\mat{M^{\max}}$ has no positive off-diagonal elements, as all of the elements of $\mat{P^{\max}}$ are between $0$ and $1$. So, the only thing we need to check is that $\mat{M^{\min}}$ is an M-matrix.

    From the definition of an M-matrix \cite{berman1994nonnegative}, $\mat{M^{\min}}$ is an M-matrix if and only if $1 \geq \rho(\lambda \mat{P^{\max}}) = \lambda \rho(\mat{P^{\max}})$. This completes the proof.
\end{proof}

\section{Gauss-Seidel Method for Interval Matrices}
\label{sec:app:gauss_seidel}

Our exposition here follows section 5.7 of \citet{horavcek2019interval}. Suppose we have a linear system $\mat{A} \mat{x} = \mat{b}$, where $\mat{A}$ is an interval matrix and $\mat{b}$ is an interval vector. The set of $\mat{x}$ that satisfies this system is referred to as the solution set. It is in general complicated, so we hope to find a rectangular enclosure of the solution set. The smallest such enclosure is referred to as the hull of the solution set.

Suppose we have an initial enclosure of the solution set, $\mat{x}^{(0)}$. In general this takes some effort to obtain, but we obtained this for our specific problems. Then, the interval Gauss-Seidel method proceeds as follows:

\begin{enumerate}
    \item For each variable $x_i$, compute a new enclosure of the solution set as:
    \begin{align*}
        y_i^{(k+1)} = \frac{1}{A_{ii}} \left( b_i - \sum_{j < i} A_{ij} x_j^{(k+1)} - \sum_{j > i} A_{ij} x_j^{(k)} \right)
    \end{align*}
    \item Update the enclosure of the solution set: $\mat{x}^{(k+1)} = \mat{x}^{(k)} \cap \mat{y}^{(k+1)}$.
    \item Repeat steps 1 and 2 until stopping criterion is met (either a maximum number of iterations or the difference in subsequent enclosures is less than some tolerance).
\end{enumerate}

All operations above are interval arithmetic operations. For real intervals, we have: $[a,b] + [c,d] = [a+c, b+d]$, $[a,b] - [c,d] = [a-d, b-c]$, $[a,b] \cdot [c,d] = [\min(ac, ad, bc, bd), \max(ac, ad, bc, bd)]$. Division is defined as multiplication by the reciprocal, i.e., $[a,b] / [c,d] = [a,b] \cdot [1/d, 1/c] = [\min(a/c, a/d, b/c, b/d), \max(a/c, a/d, b/c, b/d)]$, if $0 \notin [c,d]$.

Note that the zero division issue is not a concern for us, as the diagonal elements of $\mat{I} - \lambda \mat{P}$, in the total reward setting, have lower bounds always strictly positive, due to the discount factor $\lambda \in (0, 1)$. In the reachability and hitting time settings, the diagonal elements of $\mat{I} - \mat{Q}$ have lower bounds always strictly positive, due to the strict substochastic property of $\mat{Q}$. In practice, we use a small numerical offset to handle this.

\section{Solving other formulations}
\label{sec:app:solving_other_formulations}

\subsection{Reachability and hitting time}

The reachability and hitting time formulations are in Appendix~\ref{sec:app:other_formulations}. The key difference that affects our theoretical results and algorithms is that the matrix $\mat{Q}$, which is a submatrix of $\mat{P}$ restricted to the transient states, is strictly substochastic, and we don't have a discount factor $\lambda$ anymore. This introduces some technical difficulties in guaranteeing invertibility of $\mat{I} - \mat{Q}$. However, all our theoretical results all follow through with minor modifications.

\begin{lemma}
    For a strictly row-substochastic matrix $\mat{Q}$ with all elements in $[0, 1]$:
    \begin{enumerate}
        \item The spectral radius of $\mat{Q}$ is strictly less than $1$.
        \item The matrix $(\mat{I} - \mat{Q})$ is a (non-singular) M-matrix.
        \item The matrix $(\mat{I} - \mat{Q})^{-1}$ has all non-negative elements.
        \item The row sums of $(\mat{I} - \mat{Q})^{-1}$ are at most $1/(1-\alpha)$, where $\alpha = \max_i \sum_j Q_{ij}$ is the maximum row sum of $\mat{Q}$.
        \item Each element of $(\mat{I} - \mat{Q})^{-1}$ is in the interval $[0, 1/(1-\alpha)]$.
    \end{enumerate}
\end{lemma}
\begin{proof}
    The row sum of $\mat{Q}$ is strictly less than $1$, so, by a well-known result in linear algebra, the spectral radius of $\mat{Q}$ is strictly less than $1$. It follows that $\mat{I} - \mat{Q}$ is a (non-singular) M-matrix by the definition of an M-matrix \cite{berman1994nonnegative}, and hence that the inverse has all non-negative elements by Theorem 2.3 in Chapter 6 of \citet{berman1994nonnegative}.
\end{proof}

Note that the maximum row sum $\alpha$ is strictly less than $1$, so the formula is well-defined. We use this result to obtain the appropriate initial bounds on $\mat{v}$ for the reachability and hitting time formulations.

\begin{corollary}
    Given a Markov reachability problem with element-wise bounds on $\mat{Q}$ and $\mat{R}$, as defined in Appendix~\ref{sec:app:other_formulations}, the initial bounds on $\mat{v} = (\mat{I} - \mat{Q})^{-1} \mat{R} \mat{1}$ are:
    \begin{align*}
        0 \leq v_i \leq \frac{1}{1-\alpha} \max_j \sum_k R_{jk}^{\max}
    \end{align*}
    where $\alpha = \max_i \sum_j Q_{ij}^{\max}$ is the maximum row sum of $\mat{Q}$.
\end{corollary}
\begin{proof}
    First, note that the vector $\mat{R} \mat{1}$ has elements which are the row sums of $\mat{R}$. So, for each $i$, $\sum_j R_{ij}^{\min} \leq (\mat{R} \mat{1})_i \leq \sum_j R_{ij}^{\max}$. Next, noting that the row sums of $(\mat{I} - \mat{Q})^{-1}$ are at most $1/(1-\alpha)$, we have, by a similar argument as in the proof of Lemma~\ref{lem:v_bounds}, that the upper bound on $v_i$ is:
    \begin{align*}
        v_i \leq \frac{1}{1-\alpha} \max_j \sum_k R_{jk}^{\max}
    \end{align*}

    For the lower bound, unlike the argument in Lemma~\ref{lem:v_bounds}, we don't know what the row sums of $(\mat{I} - \mat{Q})^{-1}$ actually are, so we use the smallest possible row sum, which is $0$. So, we have:
    \begin{align*}
        v_i \geq 0
    \end{align*}
\end{proof}

\begin{corollary}
    Given a Markov hitting time problem with element-wise bounds on $\mat{Q}$, as defined in Appendix~\ref{sec:app:other_formulations}, the initial bounds on $\mat{v} = (\mat{I} - \mat{Q})^{-1} \mat{1}$ are:
    \begin{align*}
        0 \leq v_i \leq \frac{1}{1-\alpha}
    \end{align*}
    where $\alpha = \max_i \sum_j Q_{ij}^{\max}$ is the maximum row sum of $\mat{Q}$.
\end{corollary}
\begin{proof}
    The bound simply follows from the fact that $\mat{v}$ is the row sums of $(\mat{I} - \mat{Q})^{-1}$.
\end{proof}

Note, however, that we \emph{do not} know what the row sums of $\mat{Q}$ are, because it is not fixed. However, in practice, in order to assert the strict row-substochastic property, we use a small numerical offset, which is what we use in our implementation. Specifically, we set the row sums of $\mat{Q}$ to be at most $1 - \epsilon$, where $\epsilon$ is by default set to $10^{-6}$ in our software implementation. Based on this, we can set $\alpha$ to be $1 - \epsilon$.

Next, we can tighten the bounds on $\mat{v}$ by forming the appropiate linear systems: $(\mat{I} - \mat{Q}) \mat{v} = \mat{R} \mat{1}$ for reachability and $(\mat{I} - \mat{Q}) \mat{v} = \mat{1}$ for hitting time. In both cases, though $\mat{Q}$ is a submatrix of $\mat{P}$, it is square, so we can analyze the similar interval M-matrix condition for the Gauss-Seidel method to achieve the hull of the solution set.

\begin{corollary}
    Let $\mat{Q}$ be a row-substochastic matrix bounded element-wise by $\mat{Q^{\min}}$ and $\mat{Q^{\max}}$. Then, it is an interval M-matrix if and only if $\rho(\mat{Q^{\max}}) \leq 1$.
\end{corollary}
\begin{proof}
    The proof is identical to that of Theorem~\ref{thm:interval_m_matrix}, but without the discount factor.
\end{proof}

Based on the above extensions of our theoretical results, we now have similar algorithms for reachability and hitting time. As usual, we first optimize for bounds on $\mat{\theta}$, then we propagate to $\mat{Q}$ and $\mat{R}$ (if applicable), obtain initial bounds on $\mat{v}$ as above, and then apply the Gauss-Seidel method to tighten the bounds on $\mat{v}$.

\subsection{Feasibility}
Our method easily extends to the feasibility version of all three problems. We still obtain all the necessary bounds as in the optimization version. At the end, when we would optimize for the quantity of interest, we simply check if the overall problem is feasible, given the bounds on the quantity of interest.

\section{Special Cases}
\label{sec:app:special_cases}

Here we enumerate the cases for when one or more of $\mat{\pi}, \mat{P}, \mat{r}$ is fixed.

\begin{table}[!htb]
    \centering
    \begin{tabular}{c|c|c|l}
         Fixed Variable(s) & Objective & Constraint & Explanation \\ \hline
         $\mat{\pi}$      & Linear             & Bilinear           & Only objective affected \\
         $\mat{P}$         & Bilinear           & Linear             & Only constraint affected \\
         $\mat{r}$         & Bilinear           & Bilinear           & Constraint affected but bilinear term remains \\ 
         \hline
         $\mat{\pi}, \mat{P}$ & Linear           & Linear             & Objective and constraint affected \\
         $\mat{\pi}, \mat{r}$ & Linear           & Bilinear           & Objective and constraint affected, but bilinear term remains \\
         $\mat{P}, \mat{r}$ & Linear             & Eliminated             & $\mat{v}$ fully determined, leaving objective linear in $\mat{\pi}$
    \end{tabular}
    \caption{Special Cases for Optimization Problem. Objective refers to the objective $\mat{\pi}^\top \mat{v}$ and constraint refers to the constraint $\mat{v} = \lambda \mat{P} \mat{v} + \mat{r}$.}
    \label{tab:special_cases}
\end{table}

\section{Implementation Details}
\label{sec:app:implementation}

We release \texttt{markovml} along with this paper, with full documentation and tutorials. We also include the code to reproduce our experiments. Below we discuss some important features.

\subsection{Supported Models}

Currently, the implementation supports the following models from \texttt{sklearn}:

\begin{itemize}
    \item Linear regression (\texttt{sklearn.linear\_model.LinearRegression})
    \item Ridge regression (\texttt{sklearn.linear\_model.Ridge})
    \item Lasso regression (\texttt{sklearn.linear\_model.Lasso})
    \item Logistic regression (\texttt{sklearn.linear\_model.LogisticRegression})
    \item Decision tree regression (\texttt{sklearn.tree.DecisionTreeRegressor})
    \item Decision tree classifier (\texttt{sklearn.tree.DecisionTreeClassifier})
    \item Random forest regression (\texttt{sklearn.ensemble.RandomForestRegressor})
    \item Random forest classifier (\texttt{sklearn.ensemble.RandomForestClassifier})
    \item Multi-layer perceptron regression (\texttt{sklearn.neural\_network.MLPRegressor})
    \item Multi-layer perceptron classifier (\texttt{sklearn.neural\_network.MLPClassifier})
\end{itemize}

From \texttt{pytorch}, we support neural networks build as \texttt{pytorch.nn.Sequential} models with ReLu or linear layers, with possibly a softmax layer at the end for classifiers.

Lastly, we implemented a model called \texttt{DecisionRules}, which allows the user to specify a series of ``if-then" rules specified in natural language, e.g., ``if age >= 65 then 0.8". This enables encoding tables from the literature, like age-stratified mortalities. Our code parses these strings into an internal representation.

\subsection{MILP formulations}

We make use of the MILP formulations from \texttt{gurobi-machinelearning}. At time of writing, \texttt{gurobi-machinelearning} did not support the softmax function, although it did support the logistic function. Hence, for multilayer perceptron classifiers from \texttt{sklearn} and for neural network classifers from \texttt{torch}, we implemented the MILP formulation ourselves. We also implemented the MILP formulation of our \texttt{DecisionRules} model through a series of logical implications.

\subsection{Code Example}

In only a few lines, a user can build a Markov process, integrate an ML model, and optimize the reward. 

\begin{lstlisting}[style=pythonstyle, caption=Optimizing the total reward of a Markov process with parameters given by a logistic regression model]
from markovml.markovml import MarkovReward
from sklearn.linear_model import LogisticRegression
import numpy as np

# Create a Markov process with 2 states and 2 features
mrp = MarkovReward(n_states=2, n_features=2)

# fix some of the parameters
mrp.set_r([1, 0])
mrp.set_pi([1, 0])

# train a classifier
X = np.random.rand(100, 2)
y = np.random.randint(0, 2, 100)
clf = LogisticRegression().fit(X, y)

# add the classifier to the Markov process
mrp.add_ml_model(clf)

# link it to the transition probabilities
# this means the first output of the first ML model is the probability of transitioning to the second state
mrp.set_P([[1 - mrp.ml_outputs[0][0], mrp.ml_outputs[0][0]], [0, 1]])

# set bounds on the features
mrp.add_feature_constraint(mrp.features[0] >= 65)
mrp.add_feature_constraint(mrp.features[1] >= 100)

# optimize the reward
mrp.optimize(sense="max")
\end{lstlisting}

All the necessary steps, like checking validity of values, adding the MILP formulation of the ML model to the optimization problem, as well as running our decomposition and bound propagation scheme, are handled in the background by our package.

\section{Numerical experiments setup}
\label{sec:app:num_exp}

For all experiments, we fix the number of features $m=5$, the feature set $\mathcal{X} = [-1, 1]^5$, and the discount factor $\lambda = 0.97$. We train ML models on randomly generated data ($10,000$ random points). For regression models, we draw $\mat{X} \sim \mathcal{N}(0,1)^{10000 \times 5}$, $\mat{\beta} \sim \mathcal{N}(0,1)^5$, and then generate data points as $\mat{X} \mat{\beta} + \mat{\epsilon}$, where $\mat{\epsilon} \sim \mathcal{N}(0,1)^{10000}$. For classification models, we use the same linear form but draw data points from a Bernoulli distribution with probabilities $\sigma(\mat{X} \mat{\beta})$, where $\sigma(\cdot)$ is the logistic function.

For rewards, we train a regression model whose output is set to $r_1$. For $i=2, \ldots, n$, we set $r_i$ to be output divided by $i$. For probabilities, we train binary classifiers for $\mat{\pi}$ and rows of $\mat{P}$. For $\mat{\pi}$, we set $\pi_1$ equal to the classifier output, the remaining probability is assigned to $\pi_2$, and the remaining are set to zero. For rows of $\mat{P}$, the classifier output is the probability of remaining in the same state, with the remaining probability assigned to transitioning to the next state, and the rest of the row is $0$. We make the last state absorbing, i.e., transitions to itself with probability $1$. The Markov process is essentially a discrete-time birth process, as one can only transition to the next state, or stay in the same state, with the final state being absorbing.

For each configuration, we run 10 instances, each with random data, train the models on the random data, and solve the optimization problem directly with an out-of-the-box solver and with our method.

We perform the following experiments:
\begin{enumerate}[noitemsep, topsep=0pt]
    \item \textit{State Space:} Vary $n \in \{5,10,20,50,100,200\}$ using three models: a linear regression for $\mat{r}$ and two logistic regressions for $\mat{\pi}$ and the first row of $\mat{P}$ (the remaining rows are uniform).
    \item \textit{Model Count:} Fix $n=20$ with linear regression for $\mat{r}$ and logistic regression for all probability models. Models for $\mat{r}$ and $\mat{\pi}$ are always used, while the number of modeled rows in $\mat{P}$ varies from $1$ to $19$, so that the total number of models runs from $3$ to $21$.
    \item \textit{Decision Tree Depth:} Fix $n=20$ using three decision tree models (for $\mat{r}$, $\mat{\pi}$, and the first row of $\mat{P}$) with depths in $\{2,3,4,6,8\}$. Remaining rows of $\mat{P}$ are uniform.
    \item \textit{Neural Network Architecture:} Fix $n=20$ using three ReLU multilayer perceptrons, for $\mat{r}$, $\mat{\pi}$, and the first row of $\mat{P}$. We vary hidden layers in $\{1,2\}$ and neurons per layer in $\{5,10,15,20\}$, with remaining rows of $\mat{P}$ set uniformly.
\end{enumerate}

We record runtime, objective value, and optimizer status. For our method, the runtime includes solving the smaller MILPs, bound computations, the interval Gauss-Seidel phase, and the final optimization. All experiments are run on an Intel Core i7 with 6 cores using Gurobi Optimizer 12.0.0 \cite{gurobi} (1200 second limit per problem, and presolve disabled throughout). All code and instructions on how to reproduce our experiments is included along with this paper.

We computed the geometric mean and geometric standard deviations of runtimes, and proportions of optimizer statuses. For the runtime analysis, we only kept instances that were solved to optimality. Sometimes, the optimizer would return a suboptimal status, meaning it cannot prove optimality. If for an instance, our method did converge to optimality, and the direct method returned a suboptimal status, and the objective values were the same up to a tolerance of $10^{-12}$, we consider the suboptimal status to be optimal for the runtime statistical analysis. We performed paired t-tests on the log of runtimes (for instances where both methods returned optimal) and $\chi^2$ tests for statuses.

\newpage
\section*{NeurIPS Paper Checklist}

The checklist is designed to encourage best practices for responsible machine learning research, addressing issues of reproducibility, transparency, research ethics, and societal impact. Do not remove the checklist: {\bf The papers not including the checklist will be desk rejected.} The checklist should follow the references and follow the (optional) supplemental material.  The checklist does NOT count towards the page
limit. 

Please read the checklist guidelines carefully for information on how to answer these questions. For each question in the checklist:
\begin{itemize}
    \item You should answer \answerYes{}, \answerNo{}, or \answerNA{}.
    \item \answerNA{} means either that the question is Not Applicable for that particular paper or the relevant information is Not Available.
    \item Please provide a short (1–2 sentence) justification right after your answer (even for NA). 
\end{itemize}

{\bf The checklist answers are an integral part of your paper submission.} They are visible to the reviewers, area chairs, senior area chairs, and ethics reviewers. You will be asked to also include it (after eventual revisions) with the final version of your paper, and its final version will be published with the paper.

The reviewers of your paper will be asked to use the checklist as one of the factors in their evaluation. While "\answerYes{}" is generally preferable to "\answerNo{}", it is perfectly acceptable to answer "\answerNo{}" provided a proper justification is given (e.g., "error bars are not reported because it would be too computationally expensive" or "we were unable to find the license for the dataset we used"). In general, answering "\answerNo{}" or "\answerNA{}" is not grounds for rejection. While the questions are phrased in a binary way, we acknowledge that the true answer is often more nuanced, so please just use your best judgment and write a justification to elaborate. All supporting evidence can appear either in the main paper or the supplemental material, provided in appendix. If you answer \answerYes{} to a question, in the justification please point to the section(s) where related material for the question can be found.

IMPORTANT, please:
\begin{itemize}
    \item {\bf Delete this instruction block, but keep the section heading ``NeurIPS Paper Checklist"},
    \item  {\bf Keep the checklist subsection headings, questions/answers and guidelines below.}
    \item {\bf Do not modify the questions and only use the provided macros for your answers}.
\end{itemize}


\begin{enumerate}

\item {\bf Claims}
    \item[] Question: Do the main claims made in the abstract and introduction accurately reflect the paper's contributions and scope?
    \item[] Answer: \answerYes
    \item[] Justification: We state the key motivation (Markov models with embedded ML models), the key idea (bilinear programming) and the key result (massive speedups) in the abstract.

\item {\bf Limitations}
    \item[] Question: Does the paper discuss the limitations of the work performed by the authors?
    \item[] Answer: \answerTODO{} 
    \item[] Justification: \justificationTODO{}

\item {\bf Theory assumptions and proofs}
    \item[] Question: For each theoretical result, does the paper provide the full set of assumptions and a complete (and correct) proof?
    \item[] Answer: \answerYes
    \item[] Justification: All theorems and lemmas are clearly stated, with all proofs given in the Appendix. For proofs using more novel techniques (e.g. the work on interval Gauss-Seidel), we provide informal discussion in the main body of the paper.

    \item {\bf Experimental result reproducibility}
    \item[] Question: Does the paper fully disclose all the information needed to reproduce the main experimental results of the paper to the extent that it affects the main claims and/or conclusions of the paper (regardless of whether the code and data are provided or not)?
    \item[] Answer: \answerYes
    \item[] Justification: Details of the numerical experiments are given in the Appendix. As well, we release all code for the numerical experiments (with instructions on how to run each of the experiments with a simple command-line command). We also release the code for the case study, as a Jupyter notebook, and the machine learning model used. All of this accompanies our main software contribution, \texttt{markovml}. 

\item {\bf Open access to data and code}
    \item[] Question: Does the paper provide open access to the data and code, with sufficient instructions to faithfully reproduce the main experimental results, as described in supplemental material?
    \item[] Answer: \answerYes
    \item[] Justification: We release \texttt{markovml} in the supplementary materials, which is our software package. We also release all the code for the numerical experiments, as simple scripts to run them, and for the case study. We cannot release patient data for the case study, but we do release the machine learning model trained on them, which is all that is needed to perform the case study.

\item {\bf Experimental setting/details}
    \item[] Question: Does the paper specify all the training and test details (e.g., data splits, hyperparameters, how they were chosen, type of optimizer, etc.) necessary to understand the results?
    \item[] Answer: \answerYes
    \item[] Justification: Described roughly in the main paper and fully in the appendix.

\item {\bf Experiment statistical significance}
    \item[] Question: Does the paper report error bars suitably and correctly defined or other appropriate information about the statistical significance of the experiments?
    \item[] Answer: \answerYes
    \item[] Justification: Plots show statistical significance, using either t-tests or $\chi^2$ tests, depending on the data.

\item {\bf Experiments compute resources}
    \item[] Question: For each experiment, does the paper provide sufficient information on the computer resources (type of compute workers, memory, time of execution) needed to reproduce the experiments?
    \item[] Answer: \answerYes
    \item[] Justification: Described in appendix. Also, with our experimental code, we also release a ``short" version of the experiments which can be run in less time, so that an interested reader can see the crux of the experimental results quickly.
    
\item {\bf Code of ethics}
    \item[] Question: Does the research conducted in the paper conform, in every respect, with the NeurIPS Code of Ethics \url{https://neurips.cc/public/EthicsGuidelines}?
    \item[] Answer: \answerYes
    \item[] Justification: We conform in every respect with the NeurIPS Code of Ethics.

\item {\bf Broader impacts}
    \item[] Question: Does the paper discuss both potential positive societal impacts and negative societal impacts of the work performed?
    \item[] Answer: \answerYes
    \item[] Justification: In the Discussion we discuss its role in safety-critical domains.
    
\item {\bf Safeguards}
    \item[] Question: Does the paper describe safeguards that have been put in place for responsible release of data or models that have a high risk for misuse (e.g., pretrained language models, image generators, or scraped datasets)?
    \item[] Answer: \answerNA
    \item[] Justification: None needed.

\item {\bf Licenses for existing assets}
    \item[] Question: Are the creators or original owners of assets (e.g., code, data, models), used in the paper, properly credited and are the license and terms of use explicitly mentioned and properly respected?
    \item[] Answer: \answerNA
    \item[] Justification: No existing assets used.

\item {\bf New assets}
    \item[] Question: Are new assets introduced in the paper well documented and is the documentation provided alongside the assets?
    \item[] Answer: \answerYes
    \item[] Justification: Our software, \texttt{markovml}, is fully documented with extensive examples.

\item {\bf Crowdsourcing and research with human subjects}
    \item[] Question: For crowdsourcing experiments and research with human subjects, does the paper include the full text of instructions given to participants and screenshots, if applicable, as well as details about compensation (if any)? 
    \item[] Answer: \answerNo
    \item[] Justification: No human subjects.

\item {\bf Institutional review board (IRB) approvals or equivalent for research with human subjects}
    \item[] Question: Does the paper describe potential risks incurred by study participants, whether such risks were disclosed to the subjects, and whether Institutional Review Board (IRB) approvals (or an equivalent approval/review based on the requirements of your country or institution) were obtained?
    \item[] Answer: \answerNA
    \item[] Justification: No human subjects.

\item {\bf Declaration of LLM usage}
    \item[] Question: Does the paper describe the usage of LLMs if it is an important, original, or non-standard component of the core methods in this research? Note that if the LLM is used only for writing, editing, or formatting purposes and does not impact the core methodology, scientific rigorousness, or originality of the research, declaration is not required.
    \item[] Answer: \answerNo
    \item[] Justification: Does not use LLMs.

\end{enumerate}

\end{document}